%% file: projective-reconstruction_v3.tex
\documentclass{amsart}
\usepackage{microtype,amssymb,algorithm,algpseudocode,graphicx,color,enumerate,cite}

\def\RR{\ensuremath{\mathbb{R}}}

\newcommand{\PP}{{\mathbb P}}

\newcommand{\rank}{\textup{rank}}

\newtheorem{theorem}{Theorem}
\numberwithin{theorem}{section} 
\numberwithin{equation}{section} 
\newtheorem{definition}[theorem]{Definition}
\newtheorem{lemma}[theorem]{Lemma}

\newcommand{\R}{{\mathbb{R}}}

\newcommand{\dis}{\displaystyle}

\newcommand{\wh}{\widehat}

\newsavebox{\smlmat}
\savebox{\smlmat}{$\left(\begin{smallmatrix}1 & 1 & 1 \\ 0 & 1 & 1 \\ 1 & 3 & 3\end{smallmatrix}\right)$}

\title{On the existence of a projective reconstruction}
\author{Hon Leung Lee}
\address{Hon Leung Lee, Mathematics, University of Washington, Seattle, WA 98195}
\email{hllee@uw.edu}

\begin{document}

\maketitle

\begin{abstract}
In this note we 
study the connection between the existence of a projective reconstruction and the existence of a fundamental matrix satisfying the epipolar constraints. 
\end{abstract}

\section{Introduction} 
Let a set of point correspondences $(x_i,y_i)\in \RR^2 \times \RR^2$ ($i=1,\ldots,m$) be given.
Consider the following three statements:
\begin{enumerate}
\item[(A)]
$(x_i,y_i)$
are the images of $m$ points in $\R^3$ in two uncalibrated cameras.
\item[(B)]
$(x_i,y_i)$
are the images of $m$ points in $\PP^3$ in two uncalibrated cameras.
\item[(C)]
There exists 
a fundamental matrix $F$ such that the $(x_i,y_i)$ satisfy the epipolar constraints.
\end{enumerate}
An in-depth study of (C) can be found in \cite{epipolarPaper_v2}.
The goal of this note is to understand the connection among these three statements. 
In the following we summarize our contribution. All the results are proved using just linear algebra.
\begin{enumerate}
\item
A standard result in two-view geometry \cite[\S 9.2]{hartley-zisserman-2003}  states that 
(A) implies (C). 
In \cite{hartley-zisserman-2003} this result was proved by classical projective geometry and drawing pictures. We offer  a modern, more rigorous, and  linear algebraic proof; see Theorem \ref{thm:fromSceneToEpipolar}.

\item
It is clear that (A) implies (B). We will show (A) and (B) are indeed equivalent; see Theorem \ref{thm:equivReconstruction}. 
The proof is based on constructing an appropriate projective transformation.

\item
We show that (C) implies (A) after making an additional assumption about the point pairs $(x_i,y_i)$.
Indeed, if (C) holds,  one can construct a pair of uncalibrated cameras $P_1,P_2$ associated to $F$. 
If we assume that $x_i$ is an epipole of $P_1$ if and only if $y_i$ is an epipole of $P_2$, then (A) holds. 
This assumption is also necessary for (A) to hold. As a result, we know (A) holds if and only if (C) and this assumption hold. This is the main theorem of this note;  see Theorem \ref{thm:equivalence}.

\end{enumerate}

In Section \ref{s2} we introduce the notation and definitions that will be used. 
In Section \ref{s3} we discuss projective reconstruction using finite, infinite, coincident and non-coincident cameras.
Finally we provide a proof of the main theorem using linear algebra, in Section \ref{s4}.

\section{Notation and definitions} \label{s2}

To begin with, we introduce the notation and definitions that will be used in this note; see
\cite{hartley-zisserman-2003}.

Denote the $n$-dimensional real projective space by $\PP^n$. For any $x,y\in \PP^n$, we say $x\sim y$ if there exists $\lambda\in \RR\setminus\{0\}$ such that $x=\lambda y$. 
The set of $m\times n$ matrices with entries in $\RR$ is denoted by $\RR^{m \times n}$, and by 
$\PP^{m\times n}$ if the matrices are only up to scale.
For $v \in \RR^3$,
$$
[v]_\times \,\,:= \,\,
\begin{pmatrix}
0 & -v_3 & v_2\\
v_3 & 0 & -v_1\\
-v_2 & v_1 & 0
\end{pmatrix}
$$
is a skew-symmetric matrix whose rank is two unless $v=0$. Also, $[v]_\times w = v \times w$, where $\times$ denotes the vector cross product.
For any $x\in \RR^n$ the symbol $\wh{x}$ denotes $(x,1)^\top$ in $ \PP^n$. 
A point in $\PP^n$ is called {\em finite} if it can be identified with $(x,1)^\top$ for some $x\in \RR^n$.

A {\em (projective) camera} can be modeled by a matrix $P \in \PP^{3 \times 4}$ with $\rank(P)=3$. Partitioning a camera as $P=\begin{pmatrix} A & b \end{pmatrix}$ where $A \in
\RR^{3 \times 3}$, we say that $P$ is a {\em finite camera} if $A$ is
nonsingular.  The camera center of $P$ is $(-A^{-1}b,1)^\top$ if $P$ is finite; and $(w,0)^\top$ otherwise, where $w$ lies in the kernel of $A$.
Two cameras $P_1,P_2$ with camera centers $c_1,c_2$ are {\em coincident} if $c_1 \sim c_2$.
A tuple $(P_1,P_2,\{w_i\}_{i=1}^m)$ is called a {\em (projective) reconstruction} of $\{(x_i,y_i)\}_{i=1}^m\subseteq \RR^2 \times \RR^2$ if 
$P_1$ and $P_2$ are projective cameras, $w_i\in \PP^3$ and  
$$
P_1 w_i \sim \wh{x}_i, \ P_2 w_i \sim \wh{y}_i  \   \text{ for all }i = 1, \ldots,m.
$$
If in addition, $P_1,P_2$ are finite cameras and $w_i$ are finite points for all $i$, then $(P_1,P_2,\{w_i\})$ is called a {\em finite (projective) reconstruction}.

A real $3\times 3$ matrix $F$ is a {\em fundamental matrix} associated to $\{(x_i,y_i)\}$ if $F$ has rank two  and 
the following {\em epipolar constraints} hold:
\begin{equation*}
\wh{y}_i^\top F \wh{x}_i = 0 \ \text{ for any }i.
\end{equation*}

\section{Projective reconstruction} \label{s3}
\label{sec:PR}
Given point correspondences $\{ (x_i, y_i) \in \RR^2 \times \RR^2, \,\,i=1,\ldots,m \}$, the {\em projective reconstruction problem} is to decide if 
there is a projective reconstruction of these point pairs, and the {\em finite projective reconstruction problem} is to determine if the pairs admit a 
finite projective reconstruction. We first show that these two problems, as well as two others that naturally interpolate between them, 
are all equivalent. 

\begin{theorem}\label{thm:equivReconstruction}
Let $\{ (x_i,y_i) \in \RR^2 \times \RR^2, \,\,i = 1,\ldots,m \}$ be given. 
Then the  following statements are equivalent:
\begin{enumerate}

\item \label{P3Scene3}
There are cameras $P_1,P_2$ and points 
$w_i \in \PP^3, \,\,i = 1, \ldots, m$, such that $(P_1,P_2,\{w_i\})$ is a reconstruction of 
$\{(x_i,y_i)\}$.

\item \label{finiteCam3}
There are {\sc finite} cameras  $P_1,P_2$ and points 
$w_i \in \PP^3, \,\, i=1,\ldots,m$, such that $(P_1,P_2,\{w_i\})$ is a reconstruction of 
$\{(x_i,y_i)\}$.

\item \label{finiteScene3}
There are {\sc finite} cameras $P_1,P_2$ and {\sc finite} points 
$w_i \in \PP^3, \,\,i=1,\ldots,m$, such that $(P_1,P_2,\{w_i\})$ is a reconstruction of 
$\{(x_i,y_i)\}$.

\item \label{3dScene3}
There is a {\sc finite} camera $P_2$  and {\sc finite} points 
$w_i \in \PP^3, \,\,i=1,\ldots,m$, such that $(P_1,P_2,\{w_i\})$ is a reconstruction of 
$\{(x_i,y_i)\}$, with the first camera $P_1 :=\begin{pmatrix} I & 0 \end{pmatrix}$ where $I$ is the $3 \times 3$ identity matrix.

\end{enumerate}
\end{theorem}

If $P$ is a camera matrix, there is a nonsingular matrix $H \in \RR^{4 \times 4}$ such that $PH^{-1} = \begin{pmatrix} I & 0 \end{pmatrix}$. For instance, take $H$ to be the nonsingular $4 \times 4$ matrix obtained by adding an appropriately chosen additional row to $P$. 
In order to prove Theorem~\ref{thm:equivReconstruction}, we will first need the following simple fact 
that for any finite collection of nonzero points in $\RR^n$, there is always a hyperplane through the origin that avoids all of them.

\begin{lemma} \label{lem:avoidence}
Given $v_1,\ldots,v_m\in \RR^n \setminus\{0\}$, there exists $a\in \RR^n$ such that 
$a^\top v_i \neq 0$ for all $i$. 
\end{lemma}

\begin{proof}
Let $S:= \{v_1,\ldots,v_m\}$. We want to show that there exists  $a\in \RR^n$ such that 
$a^\perp \cap S =  \emptyset$.
Suppose to the contrary, for any $a\in \RR^n$ one has $a^\perp \cap S \neq \emptyset$. 
Then $a\in v_i^\perp$ for some $i$. Thus $\RR^n = v_1^\perp \cup \cdots \cup v_m^\perp$ which implies that $\RR^n = v_i^\perp$ for some $i$, and hence,
this $v_i=0$. This contradicts our assumption.
\end{proof}

We now come to the key ingredient in the proof of Theorem~\ref{thm:equivReconstruction} which allows us to always 
replace a projective reconstruction with a finite projective reconstruction whenever the first camera is of the form $\begin{pmatrix} I & 0 \end{pmatrix}$.

\begin{lemma} \label{lem:infToFin} Given point pairs $ \{ (x_i,y_i) 
\in \RR^2\times \RR^2, \,\,i = 1,\ldots,m \}$, suppose we have cameras $P_1 = \begin{pmatrix} I & 0 \end{pmatrix}$ and $P_2 = \begin{pmatrix} A & b \end{pmatrix}$, a set  $\sigma\subseteq \{1,\ldots,m\}$,  and points $v_i\in \RR^3$, $i = 1,\ldots,m$ such that:
\begin{center}
$\begin{array} {lllcl}\label{eq:P1Inf} 
\forall \,\,i \in \sigma,  &  \,\,\,\, v_i \neq 0, &  \,\,\,\,P_1 \begin{pmatrix} v_i \\ 0 \end{pmatrix}  \sim \wh{x}_i & \textup{ and } & 
P_2 \begin{pmatrix} v_i \\ 0 \end{pmatrix}  \sim \wh{y}_i;  \\
&&&\\
\forall \,\,i \not \in \sigma,  &  &  \,\,\,\,P_1 \wh{v}_i  \sim \wh{x}_i  & \textup{ and } & P_2 \wh{v}_i \sim \wh{y}_i.
\end{array}$
\end{center}

Then there exists a finite camera $P_2'$ and points $v_i' \in \RR^3$, $i=1,\ldots,m$ such that 
$(P_1,P_2', \{\wh{v}_i'\})$   is a finite reconstruction of $\{(x_i,y_i)\}$  .
In addition, if $b\neq 0$, then $P_1$ and $P_2'$ are  non-coincident cameras.
\end{lemma}

\begin{proof}
Let the camera centers of $P_1$ and $P_2$ be represented by $c_1 = \wh{0}$ and $c_2$ respectively. Since $c_1,c_2$, $(v_i^\top, 0)^\top, \,\,i \in \sigma $  and  $ \wh{v}_i, \,\, i\notin \sigma $  are all  nonzero  points in $\RR^4$, by Lemma \ref{lem:avoidence} there is a vector $a\in \RR^3$ and a scalar $\alpha \in \RR$ such that
\begin{align} \label{eq:avoidAll}
(a^\top \,\,\, \alpha) \,c_i \neq 0, \,\,i=1,2,  \,\,\,\, 
(a^\top  \,\,\, \alpha) \,\begin{pmatrix} v_i \\ 0 \end{pmatrix} \neq 0 \,\,(i \in \sigma), \,\,\, 
(a^\top \,\,\, \alpha) \,  \wh{v}_i\neq 0 \,\,(i\notin \sigma).
\end{align}

Since $(a^\top \,\,\alpha) \,c_1 \neq 0$, we have that $\alpha \neq 0$. So by scaling, we may assume that $\alpha=1$ in \eqref{eq:avoidAll}.

Consider the invertible matrix $\dis H := \begin{pmatrix} I & 0 \\ a^\top & 1\end{pmatrix}$. 
Then $H^{-1} := \begin{pmatrix} I & 0 \\ -a^\top & 1\end{pmatrix}$, and $P_1 H^{-1} = P_1$ and  $P_2 H^{-1} = \begin{pmatrix} A - ba^\top & b\end{pmatrix}$. Furthermore, 
$$H c_2 = \begin{pmatrix} \ast \\ 
\begin{pmatrix} a^\top & 1 \end{pmatrix} c_2\end{pmatrix}, \,\,\,
H\begin{pmatrix} v_i \\ 0 \end{pmatrix} = 
\begin{pmatrix} v_i \\    \begin{pmatrix} a^\top &1 \end{pmatrix} \begin{pmatrix} v_i \\ 0 \end{pmatrix} \end{pmatrix}, \,\,\,H \wh{v}_i = 
\begin{pmatrix} v_i \\  \begin{pmatrix} a^\top &1 \end{pmatrix} \wh{v}_i \end{pmatrix} $$
which are all  finite by \eqref{eq:avoidAll}.  
In particular, $P_2 H^{-1}$ is a finite camera as its center $Hc_2$ is finite. 
The proof is completed by taking $P_2' = P_2 H^{-1}$, 
$\wh{v}_i' \sim H \begin{pmatrix} v_i \\ 0 \end{pmatrix} $ ($i\in \sigma$) and 
$\wh{v}_i' \sim H \wh{v}_i $ ($i\notin \sigma$).

If we further assume $b\neq 0$, then $P_1$ and $P_2$ are non-coincident cameras. Hence $P_1 = P_1 H^{-1}$ and $P_2' = P_2 H^{-1}$ are also non-coincident.

\end{proof}

\noindent{\em Proof of Theorem~\ref{thm:equivReconstruction}}: 
Clearly, \eqref{3dScene3} $\Rightarrow$ \eqref{finiteScene3} $\Rightarrow$  \eqref{finiteCam3} $\Rightarrow$ \eqref{P3Scene3}. 
For \eqref{P3Scene3} $\Rightarrow$ \eqref{3dScene3}, let $H$ be a homography so that 
$P_1' := P_1 H^{-1} = \begin{pmatrix} I & 0 \end{pmatrix}$ and let 
$P_2' := P_2 H^{-1} = \begin{pmatrix} A&  b\end{pmatrix}$.
Then 
$(P_1', P_2', \{H w_i\})$ is a reconstruction of $\{(x_i,y_i)\}$. We can now use 
Lemma \ref{lem:infToFin} to turn this into a finite reconstruction where the first camera is still $\begin{pmatrix} I & 0 \end{pmatrix}$. Therefore, we conclude that all four statements in the theorem are equivalent. \hfill $\Box$

We now prove that the equivalences in Theorem~\ref{thm:equivReconstruction} also hold if we further require that the cameras are non-coincident (coincident) in each statement. 
 
\begin{theorem} \label{thm:non-coincident PR = non-coincident FPR}
The four statements in Theorem~\ref{thm:equivReconstruction} are equivalent if we replace ``cameras $P_1,P_2$''  in each statement with ``non-coincident cameras $P_1, P_2$''. 
\end{theorem}

\begin{proof} As before, we only need to show that \eqref{P3Scene3} $\Rightarrow$ \eqref{3dScene3}. Let $P_1' = P_1 H^{-1} = \begin{pmatrix} I & 0 \end{pmatrix}$ and $P_2' = P_2 H^{-1} = \begin{pmatrix} A&  b\end{pmatrix}$ as in the proof of this direction in Theorem~\ref{thm:equivReconstruction}. 
If $P_1$ and $P_2$ are non-coincident in \eqref{P3Scene3}, then   $P_1'$ and $P_2'$ are also non-coincident. If $A$ is nonsingular then $b\neq 0$.
 If $A$ is singular, then $b\neq 0$ because rank$(P_2') = 3$. Now using the last part of Lemma \ref{lem:infToFin}, we can turn the reconstruction $(P_1', P_2', \{H w_i\})$ into a finite reconstruction with non-coincident cameras with the first camera equal to $\begin{pmatrix} I & 0 \end{pmatrix}$. This is the statement in \eqref{3dScene3}.
\end{proof}

\begin{theorem} \label{thm:coincident PR = coincident FPR}
The four statements in Theorem~\ref{thm:equivReconstruction} are  equivalent if we replace ``cameras $P_1,P_2$''  in each statement with ``coincident cameras $P_1, P_2$''. 
\end{theorem}

\begin{proof}
Again, we only need to prove that \eqref{P3Scene3} $\Rightarrow$ \eqref{3dScene3}. If $P_1, P_2$ are coincident cameras in 
\eqref{P3Scene3}, then $P_1' = P_1 H^{-1} = \begin{pmatrix} I & 0 \end{pmatrix}$ and 
$P_2' = P_2 H^{-1} = \begin{pmatrix} A&  b\end{pmatrix}$ are also coincident. Therefore, $\wh{0}$ is their common center and hence $b=0$. This implies that $A$ is 
nonsingular since otherwise $\rank(P_2') < 3$. 
Now consider the points $w_i'$ obtained by setting the last coordinate of each $w_i$ from the reconstruction in \eqref{P3Scene3} to $1$. Then $(P_1', P_2', \{w_i'\})$ is a finite reconstruction of $\{(x_i,y_i)\}$.
\end{proof}

By the above results we can always obtain a finite projective reconstruction whenever a projective reconstruction exists. Also, if the projective reconstruction was with non-coincident (coincident) cameras there is also a  finite reconstruction with non-coincident (coincident) cameras. Further, in each case the first camera can be assumed to be $\begin{pmatrix} I & 0 \end{pmatrix}$. This understanding will be useful in the next section.

We end this section by discussing the geometry of the point pairs for which a projective reconstruction with coincident cameras exists.

\begin{definition} \label{def:projectively equivalent}
{\rm
Given $(x_i,y_i)\in \RR^2\times \RR^2$, $i = 1,\ldots,m$, 
we say that $\{x_i\}$ is 
{\em projectively equivalent} to  $\{y_i\}$
if there is a nonsingular matrix 
$H\in \RR^{3\times 3}$ 
such that 
$H\wh{x}_i \sim \wh{y}_i$ for all $1\leq i \leq m$. 
}
\end{definition}

The following result captures the close relationship between projectively equivalent point sets and projective reconstruction with coincident cameras.

\begin{theorem} \label{thm:upToNonsingular}
Let $(x_i,y_i) \in \RR^2 \times \RR^2$, $i = 1,\ldots,m$ be given. 
Then there exists a finite reconstruction $(P_1=\begin{pmatrix} I & 0 \end{pmatrix},P_2,\{\wh{w}_i\}_{i=1}^m)$ 
of $\{(x_i,y_i)\}$ where $P_1$ and $P_2$ are coincident cameras if and only if 
$\{x_i\}$ is projectively equivalent to  $\{y_i\}$.
\end{theorem}

\begin{proof}
Suppose $P_2 =  \begin{pmatrix} A & b \end{pmatrix}$.  
If $P_1$ and $P_2$ are coincident, then their common camera center is $\wh{0}$ which is finite. 
Hence $P_2$ is a finite camera and 
$b=0$.  Unwinding $P_1 \wh{w}_i \sim \wh{x}_i$ and $P_2 \wh{w}_i \sim \wh{y}_i$ we obtain 
$A\wh{x}_i \sim \wh{y}_i$ for all $i=1,\ldots,m$. 

 For the converse, suppose there exists a nonsingular matrix $H\in \RR^{3\times 3}$ such that 
$H\wh{x}_i \sim \wh{y}_i$ for all $i=1,\ldots,m$. Then setting $P_1 := \begin{pmatrix} I & 0 \end{pmatrix}$ and 
$P_2 := \begin{pmatrix} H & 0 \end{pmatrix}$, and using the notation 
$\wh{\wh{a}}$ for $(\wh{a}^\top,1)^\top$ where $a \in \RR^2$, we see that 
$(P_1,P_2,\{\wh{\wh{x}}_i\}_{i=1}^{m} )$ is a projective reconstruction of $\{(x_i,y_i)\}$ with two coincident cameras.
\end{proof}

\section{Main theorem} \label{s4}

We now come to the more general situation of reconstruction. In this case, there is a distinguished fundamental matrix associated to the point pairs coming from the cameras in the reconstruction.
We remark that the some results in this section are formally or informally stated in \cite{hartley-zisserman-2003}, but we prove them using linear algebra instead of classical projective geometry.

\begin{theorem} \label{thm:fromSceneToEpipolar}
Let $(x_i,y_i) \in \RR^2 \times \RR^2$, $i = 1,\ldots,m$ be given. 
Consider two finite cameras 
$P_1 :=\begin{pmatrix} I & 0 \end{pmatrix}$ and $P_2:= \begin{pmatrix} A & b \end{pmatrix}$.
Suppose that 
there exist $w_i \in \RR^3$ ($1\leq  i \leq m$) such that 
  $(P_1,P_2,\{\wh{w}_i\})$ is a reconstruction of $\{(x_i,y_i)\}$.
Then there is a fundamental matrix associated to  $\{(x_i,y_i)\}$.
\end{theorem}

\begin{proof} 
Suppose that $P_1$ and $P_2$ are non-coincident cameras. Since $A$ is nonsingular one has $b\neq 0$. 
 Define $F := [b]_\times A$. Since $b \neq 0$, $\rank([b]_\times) = 2$ and $\rank(F)=2$. For a fixed $i$, the relations 
$P_1  \wh{w}_i\sim \wh{x}_i$ and $P_2 \wh{w}_i \sim \wh{y}_i$ imply that 
$\lambda_i A \wh{x}_i +  b = \mu_i \wh{y}_i$ 
for some $\lambda_i \neq 0$, $\mu_i \neq 0$. 
Hence, $F$ satisfies the epipolar constraints involving $x_i$ and $y_i$:
\begin{align*}
\wh{y}_i^\top F \wh{x}_i 
& \sim ( \lambda_i  (A\wh{x}_i)^\top +  b^\top) [b]_\times A \wh{x}_i  
\sim  (A\wh{x}_i)^\top[b]_\times A \wh{x}_i = ( A \wh{x}_i)^\top (b \times  A \wh{x}_i) = 0.
\end{align*}
If $P_1$ and $P_2$ are coincident, then there is a nonsingular matrix $H$ such that $Hx_i \sim y_i$ for all $1\leq i \leq m$, by 
Theorem \ref{thm:upToNonsingular}.  Let $t$ be any nonzero vector in $\R^3$. It follows that for any $i = 1, \ldots,m$, 
$$
y_i^\top [t]_\times H  x_i \sim y_i^\top [t]_\times y_i = y_i^\top (t \times y_i) = 0. 
$$
Thus $[t]_\times H$ is a fundamental matrix associated to $\{(x_i,y_i)\}$.
\end{proof}

We now introduce a regularity condition on $\{ (x_i, y_i) \}_{i=1}^{m}$ that is necessary for the existence of a projective reconstruction with non-coincident cameras. We will see that when the point pairs $(x_i, y_i)$ are regular, 
a reconstruction with non-coincident cameras exists if and only if a fundamental matrix exists.

\begin{definition} \label{def:irregular}
Let $A\in \RR^{3\times 3}$ and $b\in \RR^3$.
We say that $(x,y)\in \RR^2\times \RR^2$ is $(A,b)$-{\em irregular} if one of the following mutually exclusive 
conditions hold:
\begin{align} \label{eq:irregularEquiv}
([b]_\times A\wh{x} = 0\text{ and }\wh{y}^\top [b]_\times \neq 0) \,\,\,\textup{ or } \,\,\,
([b]_\times A\wh{x} \neq 0 \text{ and } \wh{y}^\top [b]_\times = 0).
\end{align}
Say $(x,y)$ is $(A,b)$-{\em regular} if it is not $(A,b)$-irregular.
\end{definition}

If $(x,y)$ is an $(A,b)$-irregular pair then $\wh{y}^\top [b]_\times A \wh{x} = 0$. This implies that if 
$P_1 = \begin{pmatrix} I & 0 \end{pmatrix}$ and $P_2 = \begin{pmatrix} A & b \end{pmatrix}$ are non-coincident finite cameras then $(x,y)$ satisfies the epipolar constraint $\wh{y}^\top F \wh{x} = 0$ (where $F = [b]_\times A$) whether or not there is a reconstruction $w \in \PP^3$ of $(x,y)$. In fact, more is true.

Since $P_2 = \begin{pmatrix} A & b \end{pmatrix}$ is non-coincident with $P_1$, one has $b \neq 0$. Since the fundamental matrix $F:= [b]_\times A$ has rank two, both the left and right kernel of $F$ are one-dimensional. 
Let $e_1,e_2\in \RR^3\setminus\{0\}$ be a basis vector of the 
right and left kernel of $F$ respectively. 
Then $e_1$ is called an {\em epipole} of $P_1$ while $e_2$ is called an epipole of $P_2$. 
It is known that $P_1 c_2 \sim e_1$ and $P_2 c_1 \sim e_2$, where $c_1 = \wh{0}$  and $c_2 = (-A^{-1}b^\top, 1)^\top$ 
are the camera centres of $P_1$ and $P_2$ respectively. This implies we can take $e_1 := A^{-1}b$ and $e_2 := b$.

\begin{figure} 
\centering
\def\svgwidth{250pt}
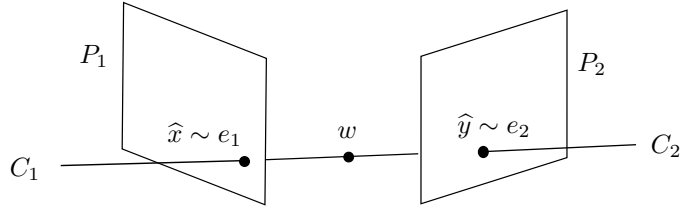
\caption{The tuple $(P_1,P_2,\wh{w})$ reconstructs $(x,y)$.}
\label{fig:epipole}  
\end{figure}

Suppose $(x,y)$ is $(A,b)$-irregular. Then as we saw earlier, 
$\wh{y}^\top F \wh{x} = 0$ holds. If $[b]_\times A \wh{x} = 0$  and $\wh{y}^\top [b]_\times \neq 0$ 
then $\wh{x}$ is an epipole of $P_1$ 
 but $\wh{y}$ is not an epipole of $P_2$. 
 If $[b]_\times A\wh{x} \neq 0 \text{ and } \wh{y}^\top [b]_\times = 0$ holds 
  then $\wh{x}$ is not an epipole of $P_1$ but $\wh{y}$ is an epipole of $P_2$.
 On the other hand, we see from Figure \ref{fig:epipole} that if  $(P_1,P_2,\wh{w})$ is a reconstruction of $(x,y)$ for some $w\in \RR^3$, and 
if $\wh{x}$ is the epipole of $P_1$, then $\wh{y}$ has to the epipole of $P_2$ (the epipoles of the two cameras lie on the line connecting the centers of the two cameras.)
This means if $(x,y)$ is $(A,b)$-irregular, then there is no finite reconstruction for $(x,y)$ using $P_1,P_2$, even though the epipolar 
constraint is trivially satisfied. This proves the following lemma.

\begin{lemma} \label{lem:regularity necessary for FPR}
Suppose that  $P_1 = \begin{pmatrix} I & 0 \end{pmatrix}$ and $P_2 = \begin{pmatrix} A & b \end{pmatrix}$ are two non-coincident finite cameras. Then, if $(x,y)$ is  $(A,b)$-irregular, 
then 
there is no $w \in \RR^3$ such that  $(P_1,P_2, \wh{w}  )$ is a reconstruction of $(x,y)$.
\end{lemma}

Notice that Lemma \ref{lem:regularity necessary for FPR} can also be verified using a simple algebraic computation, without using the notion of an epipole and the help of Figure \ref{fig:epipole}.

%

The following two  lemmas will be used to prove the main theorem of this section.

\begin{lemma} \label{lem:regular implies reconstruction}
Suppose that $P_1 = \begin{pmatrix} I & 0 \end{pmatrix}$ and $P_2 = \begin{pmatrix} A & b \end{pmatrix}$ are two non-coincident finite cameras. Then, if $(x,y)$ is  $(A,b)$-regular  and $\wh{y}^\top [b]_\times A \wh{x}=0$, 
then there exists $w \in \PP^3$  such that  $(P_1,P_2, w)$ is a reconstruction of $(x,y)$.
\end{lemma}

\begin{proof}
The assumptions about $P_1$ and $P_2$, and the equation  $\wh{y}^\top [b]_\times A \wh{x}=0$ imply $\wh{y},b,A\wh{x}$ are 
nonzero linearly dependent vectors in $\RR^3$. 
Thus there are scalars $\gamma,\beta,\alpha\in \RR$, not all zero, such that
\begin{align} \label{eq:linComb}
\gamma A\wh{x} = \beta \wh{y} - \alpha b.
\end{align} 
For a scalar $\delta$, define $w_\delta := \begin{pmatrix} \wh{x} \\ \delta \end{pmatrix} $. Then we obtain
$$
P_1 w_\delta  = \wh{x}, \  \text{ and } 
P_2 w_\delta  = A\wh{x} + \delta b.
$$
There are three cases to consider.

Case 1: $\gamma = 0$.

Then $\wh{y}\sim b$. If $A\wh{x} = 0$, then $P_2 w_\alpha = \beta \wh{y} \sim \wh{y}$ so $(P_1,P_2,w_\alpha)$ is a reconstruction of $(x,y)$. If $A\wh{x}\neq 0$, then $\wh{y} \sim A\wh{x}$ by the regularity of $(x,y)$. 
Thus $P_2 w_0 = A\wh{x} \sim \wh{y}$ so 
$(P_1,P_2,w_0)$ is a reconstruction of $(x,y)$.

Case 2: $\gamma \neq 0$ and $\beta = 0$.

In this case \eqref{eq:linComb} gives  $A\wh{x} = - \alpha b$ after scaling. If $\alpha = 0$ then $A\wh{x} = 0$ and $\wh{y}\sim b$ by the regularity of $(x,y)$. Thus $P_2 w_1 = b \sim \wh{y}$ which means $(P_1,P_2,w_1)$ is a reconstruction of $(x,y)$.
If $\alpha \neq 0$ then $A\wh{x}\neq 0$ and $A\wh{x} \sim b$. By the regularity of $(x,y)$, one has 
$\wh{y} \sim A\wh{x}$. Thus $(P_1,P_2,w_0)$ is a reconstruction of $(x,y)$.

Case 3: $\gamma \neq 0$ and $\beta \neq 0$.

\eqref{eq:linComb} implies $A\wh{x} = \beta \wh{y} - \alpha b$ after scaling. 
Hence $P_2 w_\alpha = A\wh{x} + \alpha b = \beta \wh{y} \sim \wh{y}$ which concludes that  
$(P_1,P_2,w_\alpha)$ is a reconstruction of $(x,y)$.
\end{proof}

\begin{lemma} \label{lem:rank3}
Let $F$ be a fundamental matrix and let $e_2 \in {\rm ker}(F^\top) \setminus\{0\}$. 
Define $P := \begin{pmatrix} [e_2]_\times F & e_2 \end{pmatrix}$.
Then $P$ has rank three. 
\end{lemma}

\begin{proof}
The proof can be found in  \cite[page 256]{hartley-zisserman-2003}, but we rewrite it here for the self-containedness of this note. 
Since $e_2\in{\rm ker}(F^\top) \setminus\{0\}$, we have $\rank([e_2]_\times F)=2$. It implies that the column space of $[e_2]_\times F$ is a plane in $\R^3$. Since $e_2$ is a nonzero vector orthogonal to any vector in this plane, we know $\rank(P)=3$. 
\end{proof}

We are now ready to prove the main theorem.

\begin{theorem}\label{thm:equivalence} 
Let $(x_i,y_i) \in \RR^2 \times \RR^2$, $i = 1,\ldots,m$ be given. 
Then the  following statements are equivalent:
\begin{enumerate}
\item \label{3dScene}
There exists a finite reconstruction of $\{(x_i,y_i)\}$
where one of the cameras is $\begin{pmatrix} I & 0 \end{pmatrix}$ and the two cameras are non-coincident.
\item \label{image}
There is a fundamental matrix $F$ associated to  $\{(x_i,y_i)\}$
 such that $(x_i,y_i)$ is $([e_2]_\times F, e_2)$-regular for all $i$, where 
$e_2 \in {\rm ker}(F^\top ) \setminus\{0\}$.

\end{enumerate}
\end{theorem}

\begin{proof}
First we show \eqref{image} $\Rightarrow$ \eqref{3dScene}. 
Let the matrix $F$ stated in \eqref{image}  be given. 
Notice that $\begin{pmatrix} [e_2]_\times F & e_2 \end{pmatrix}$ is a camera matrix by 
Lemma  \ref{lem:rank3}. Then, 
take $a\in \RR^3$ so that $A:= [e_2]_\times F- e_2 a^\top$ is nonsingular and 
$P_1 := \begin{pmatrix} I & 0 \end{pmatrix}$ and $P_2 := \begin{pmatrix} A & e_2 \end{pmatrix}$ are non-coincident finite cameras; see the proof of Lemma \ref{lem:infToFin} for how $a$ is chosen. 
As $[e_2]_\times A = - e_2^\top e_2 F$, one has $y_i^\top [e_2]_\times A x_i = 0$ for all $i$. 
Then  \eqref{3dScene} holds by Theorem \ref{thm:equivReconstruction} and 
Lemma  \ref{lem:regular implies reconstruction}.

Next we show the converse. Assume \eqref{3dScene} holds. 
Then there is a finite camera $P_2$ so that $P_1 :=\begin{pmatrix} I & 0 \end{pmatrix}$, $P_2$ are non-coincident cameras,  and 
there are $w_i \in \RR^3$ ($1\leq  i \leq m$) such that 
 $(P_1,P_2,\{\wh{w}_i\})$ is a reconstruction of $\{(x_i,y_i)\}$ .
We let $P_2 := \begin{pmatrix} A & b \end{pmatrix}$ where $A\in \RR^{3\times 3}$ is nonsingular and $b\in \RR^3\setminus\{0\}$. 
Consider the fundamental matrix $F:= [b]_\times A$.
By Theorem \ref{thm:fromSceneToEpipolar} and Lemma \ref{lem:regularity necessary for FPR}, 
the epipolar constraints are satisfied and each $(x_i,y_i)$ is $(A,b)$-regular.
Since $F^\top = -A^\top [b]_\times$, we have $b\in {\rm ker}(F^\top)\setminus\{0\}$.
Moreover, as $[b]_\times F = ([b]_\times)^2 A$, we know each $(x_i,y_i)$ is $([b]_\times F, b)$-regular. 
 Thus the statement \eqref{image} follows. 
\end{proof}

\bibliographystyle{plain}
\bibliography{lee}

\end{document}

%% file: epipole.pdf_tex
\begingroup%
  \makeatletter%
  \providecommand\color[2][]{%
    \errmessage{(Inkscape) Color is used for the text in Inkscape, but the package 'color.sty' is not loaded}%
    \renewcommand\color[2][]{}%
  }%
  \providecommand\transparent[1]{%
    \errmessage{(Inkscape) Transparency is used (non-zero) for the text in Inkscape, but the package 'transparent.sty' is not loaded}%
    \renewcommand\transparent[1]{}%
  }%
  \providecommand\rotatebox[2]{#2}%
  \ifx\svgwidth\undefined%
    \setlength{\unitlength}{420.4954895bp}%
    \ifx\svgscale\undefined%
      \relax%
    \else%
      \setlength{\unitlength}{\unitlength * \real{\svgscale}}%
    \fi%
  \else%
    \setlength{\unitlength}{\svgwidth}%
  \fi%
  \global\let\svgwidth\undefined%
  \global\let\svgscale\undefined%
  \makeatother%
  \begin{picture}(1,0.30874317)%
    \put(0,0){\includegraphics[width=\unitlength]{epipole.pdf}}%
    \put(-0.03149192,0.04438574){\color[rgb]{0,0,0}\makebox(0,0)[lb]{\smash{$C_1$}}}%
    \put(0.074019,0.2428959){\color[rgb]{0,0,0}\makebox(0,0)[lt]{\begin{minipage}{0.21607162\unitlength}\raggedright $P_1$\end{minipage}}}%
    \put(0.93854173,0.0796491){\color[rgb]{0,0,0}\makebox(0,0)[lb]{\smash{$C_2$}}}%
    \put(0.82746487,0.20406321){\color[rgb]{0,0,0}\makebox(0,0)[lb]{\smash{$P_2$}}}%
    \put(0.46437329,0.10182914){\color[rgb]{0,0,0}\makebox(0,0)[lb]{\smash{$w$}}}%
    \put(0.20705931,0.09606365){\color[rgb]{0,0,0}\makebox(0,0)[lb]{\smash{$\widehat{x}\sim e_1$
}}}%
    \put(0.64580725,0.10951649){\color[rgb]{0,0,0}\makebox(0,0)[lb]{\smash{$\widehat{y}\sim e_2$
}}}%
  \end{picture}%
\endgroup%